\documentclass[runningheads]{llncs}
\usepackage[T1]{fontenc}
\usepackage{graphicx}

\usepackage[misc]{ifsym}
\usepackage{booktabs}     
\usepackage{amsmath}    
\usepackage{amssymb}    
\usepackage{graphicx}   
\usepackage{float}
\usepackage{subfig}
\usepackage{xcolor}
\usepackage[figuresright]{rotating}     
\usepackage{multirow}       
\usepackage{arydshln}       
\usepackage[T1]{fontenc}

\begin{document}

\pagestyle{empty}
\authorrunning{Kailong Wu, Yule Xie, Jiaxin Ding et al.}

\title{Characterizing the Influence of Topology on Graph Learning Tasks}

\author{Kailong Wu\inst{1} \and
Yule Xie\inst{1} \and
Jiaxin Ding\inst{1}\Letter \and
Yuxiang Ren\inst{2} \and
Luoyi Fu\inst{1} \and
Xinbing Wang\inst{1} \and
Chenghu Zhou\inst{1}
}

\institute{
    Shanghai Jiao Tong University\\
    \email{\{1473686097, xyl-alter, jiaxinding, yiluofu, xwang8\}@sjtu.edu.cn, zhouchsjtu@gmail.com} \and
    2012 Laboratories, Huawei Technologies\\
     \email{renyuxiang1@huawei.com}
}

\maketitle

\begin{abstract}
   Graph neural networks (GNN) have achieved remarkable success in a wide range of tasks by encoding features combined with topology to create effective representations. However, the fundamental problem of understanding and analyzing how graph topology influences the performance of learning models on downstream tasks has not yet been well understood. In this paper, we propose a metric, TopoInf, which characterizes the influence of graph topology by measuring the level of compatibility between the topological information of graph data and downstream task objectives. We provide analysis based on the decoupled GNNs on the contextual stochastic block model to demonstrate the effectiveness of the metric. 
    Through extensive experiments, we demonstrate that TopoInf is an effective metric for measuring topological influence on corresponding tasks and can be further leveraged to enhance graph learning.
\end{abstract}

\section{Introduction}
\label{sec:introduction}

Graph neural networks (GNNs) have emerged as state-of-the-art models to learn graph representation by the message-passing mechanism over graph topology
for downstream tasks, such as node classification, link prediction, etc. \cite{GCN,GCNII,RenBZ21}. 

Despite their success, GNNs are vulnerable to the compatibility between graph topology and graph tasks \cite{luo2021learning,ESNR&GPS}. 
In essence, the effectiveness of the message-passing mechanism and, by extension, the performance of GNNs depends on the assumption that
the way information is propagated and aggregated in the graph neural networks should be conducive to the downstream tasks. 
First, if the underlying motivation of graph topology, why two nodes get connected, has no relation to the downstream task, such task-irrelevant topology can hurt the information aggregation per se.   
For example, in a random graph that lacks meaningful clusters, graph learning becomes impossible in community detection tasks. 
Furthermore, GNNs may even be outperformed by multilayer perceptrons (MLP) on graphs exhibiting heterophily \cite{CPGNN,H2GCN}.  
Second, even if the motivation for connections is compatible with tasks, graphs in real-world applications are usually inherent with ``noisy'' edges or incompleteness, due to error-prone data collection. This topological noise can also affect the effectiveness of message passing.  
Third, the vulnerability of GNNs to compatibility varies across models due to differences in their message-passing mechanisms. For example, GCN \cite{GCN} may not be as effective on heterophilic graphs. However, H2GCN \cite{H2GCN} is able to mitigate the heterophilic problem, with specially designed message passing. 
Furthermore, even the same model with different hyperparameters of network layers can also exhibit different vulnerabilities, since the reception fields of the neighborhoods are different.   
All this, there is a natural question to ask: \emph{What is the metric for the compatibility between graph topology and graph learning tasks?} 

This question, which is essential for choosing a graph learning model to deploy over a graph for downstream tasks, has not yet been systematically characterized. 
Existing work focuses on measuring homophily \cite{GeomGCNTexasAndCornell,H2GCN,measure} or edge signal-to-noise ratio \cite{ESNR&GPS}, based on the discrepancy of nodes and their neighbors in features or labels without considering graph learning models. 
Therefore, the results of these metrics can be an indicator of the performance of GCN-like models but cannot be generalized to more recently developed GNN models and guide topology modification to improve the performance of these models. 
From the above analysis, we would like to investigate
the compatibility between topology and tasks associated with graph learning models globally, and locally, the influence from each edge, either positive or negative, can be used to adjust the topology to increase or decrease the performance.  
Answering this question can help us better understand graph learning, increase the interpretability of learning models, and shed light on improving model performance.

We first propose a metric for model-dependent compatibility between topology and graph tasks, measured by the difference between the ``ideal'' results, labels, and the result of the models performing on the ``ideal'' features.    
Thereafter, we introduce a metric TopoInf,  to locally characterize the influence of each edge on the overall compatibility, by evaluating the change of such compatibility after topology perturbation. 
We provide motivating analysis based on the decoupled GNNs \cite{APPNP,PAT} on cSBMs \cite{CSBM} to validate the effectiveness of the metric and demonstrate that TopoInf is an effective metric through extensive experiments.

Our main contributions can be summarized as follows: 
\begin{itemize}
    \item  To the best of our knowledge, we are the first to measure compatibility between graph topology and learning tasks associated with graph models. 
    \item We propose a new metric, TopoInf, to measure the influence of edges on the performance of GNN models in node classification tasks, and conduct extensive experiments to validate the effectiveness. 
    \item The proposed TopoInf can be leveraged to improve performance by modifying the topology on different GNN models, and such a scheme can be applied further to different scenarios. 
\end{itemize}
\section{Preliminaries}
\label{sec:preliminaries}

\noindent\textbf{Notations. }
We denote an undirected graph as $\mathcal{G}= (\mathcal{V},\mathcal{E}, \textbf{A})$,  where $\mathcal{V}$ is the node set with cardinality $|\mathcal{V}| = n$, $\mathcal{E}$ is the edge set without self-loop, and
$\textbf{A} \in \mathbb{R}^{n \times n}$ is the symmetric adjacency matrix with $\textbf{A}_{i,j} = 1$ when $e_{ij}\in \mathcal{E}$ otherwise $\textbf{A}_{i,j} = 0$. 
$\textbf{D}$ denotes the diagonal degree matrix of $\mathcal{G}$ where $ \textbf{D}_{i, i} = d_i$, the degree of node $v_i$.
We use $\tilde{\textbf{A}} = \textbf{A} + \textbf{I}$ to represent the adjacency matrix with self-loops and $\tilde{\textbf{D}} = \textbf{D} + \textbf{I}$. The symmetric normalized adjacency matrix is $\hat{{\textbf{A}}} = {\textbf{D}}^{-1/2}\tilde{\textbf{A}}\tilde{\textbf{D}}^{-1/2}$. 
For node classification with $c$ classes, 
$\textbf{L} \in \mathbb{R}^{n \times c}$ denotes the label matrix, whose $i^{\text{th}}$ row represents the one-hot encoding of node $v_i$, 
while $\textbf{X} \in \mathbb{R}^{n\times d}$ is the feature matrix, whose $i^{\text{th}}$ row $\textbf{X}_{i, :}$ represents a $d$-dimensional feature vector of node $v_i$.

\noindent\textbf{Graph Filters.}
Recent studies show that GNN models, such as ChebNet \cite{ChebNet}, APPNP \cite{APPNP} and GPRGNN \cite{GPRGNN}, can be viewed as operations of polynomial graph spectral filters. 
Specifically, 
the effect of such GNN models on a graph signal $\textbf{x} \in \mathbb{R}^{n\times 1}$, 
can be formulated as 
graph spectral filtering operation $\boldsymbol{f}(\textbf{A})$ based on adjacency matrix $ \textbf{A} $, such that 
$
\boldsymbol{f}(\textbf{A}) \textbf{x} = \sum_{k=0}^{K} \gamma_k \hat{\textbf{A}}^k \textbf{x}, 
$
where $\hat{\textbf{A}}$ is the normalized adjacency matrix, 
$K$ is the order of the graph filter and $\gamma_k$'s are the weights, which can be assigned \cite{SGC,APPNP} or learned \cite{GPRGNN,BernNet}.
The graph filter can be extracted as the effect of topology associated with GNN models.

\noindent\textbf{Decoupled GNN. }
Recent works \cite{APPNP,NDLSLSI,PAT} show that neighborhood aggregation and feature transformation can be decoupled and formulate the graph learning tasks as 
\begin{equation}
    \hat{\textbf{L}} = \text{softmax}\left( \boldsymbol{f}(\textbf{A}) \boldsymbol{g}_{\theta}(\textbf{X}) \right)
\label{eq:decoupled_gnn}
\end{equation}
where the prediction $\hat{\textbf{L}}$ can be obtained by operating on graph feature matrix $ \textbf{X} $, 
through the feature transformation function $\boldsymbol{g}_{\theta}(\cdot)$ with learnable parameters, thereafter 
applying the graph filter $\boldsymbol{f}(\textbf{A})$ acting as the neighborhood aggregation, 
and finally passing through a softmax layer. 

\section{Methodology}
\label{sec:methodology}

In this section, we present our methodology to study the influence of topology on graph learning tasks. 
We evaluate the global compatibility between graph topology and tasks, and propose our metric, TopoInf.

\subsection{Matching graph topology and tasks}

We first relate the graph topology to the graph tasks.   
Thanks to the graph filters in the preliminaries,
the influence of the graph topology on a GNN model can be simplified and approximated as a graph filter $\boldsymbol{f}(\textbf{A})$. 
$ \boldsymbol{f}(\textbf{A}) $ can be viewed as a graph filter function $\boldsymbol{f}(\cdot): \mathbb{R}^{n \times n} \rightarrow \mathbb{R}^{n \times n} $ applied to the adjacency matrix $\textbf{A}$ of the observed graph. 
With the decoupling analysis of GNN, 
the prediction is obtained by $\boldsymbol{f}(\textbf{A})$ working on prediction results $\boldsymbol{g}_{\theta}(\textbf{X})$ based on feature matrix $ \textbf{X} $ through graph learning models \cite{DecoupledSmoothing,NDLSLSI,PAT}, where $\boldsymbol{g}_{\theta}(\textbf{X})$ on $ \textbf{X} $ can be viewed as extracting task-related information in $ \textbf{X} $ and ideally, one of the best predictions on $ \textbf{X} $ could be the label matrix $ \textbf{L} $.
Therefore, we replace the functions with $ \boldsymbol{f}(\textbf{A}) \boldsymbol{g}_{\theta}(\textbf{X}) $ 
in Equation \ref{eq:decoupled_gnn} 
to be $ \boldsymbol{f}(\textbf{A}) \textbf{L} $, which simplifies the analysis from features, makes us focus on the labels and graph tasks,
and provides an ``upper bound'' of the predicted results in the ideal situation. In the next section, we will provide a motivation analysis for this setting.  

Further,
instead of applying the softmax function, 
we perform row-normalization \cite{ESNR&GPS}, which normalizes the summation of entries on the same row to be one,   
on $\boldsymbol{f}(\textbf{A}) \textbf{L}$ as the prediction as used in previous works \cite{AllWeHaveIsLowPassFilters,ESNR&GPS}.
The row-normalized matrix $\boldsymbol{f}(\textbf{A}) \textbf{L}$ provides a summary of the edge statistics at the individual node level, as each row of $\boldsymbol{f}(\textbf{A}) \textbf{L}$ represents the label neighborhood distribution of the node, which has direct connections to GNN learning \cite{IsHomophilyaNecessityforGNN}.
As $ \text{RowNorm}(\boldsymbol{f}(\textbf{A}) \textbf{L}) = \text{RowNorm}(\boldsymbol{f}(\textbf{A})) \textbf{L}$, 
we perform row-normalization on graph filter $\boldsymbol{f}(\textbf{A})$, and $\boldsymbol{f}(\textbf{A})$ is row-normalized 
in our subsequent analysis.

Therefore, to quantify the compatibility between graph topology and graph tasks, we propose measuring the similarity between the ideal prediction
$\textbf{L}$ and the graph filter that performs the ideal feature $\boldsymbol{f}(\textbf{A}) \textbf{L}$. 
The more similar $\textbf{L}$ and $\boldsymbol{f}(\textbf{A}) \textbf{L}$ is ($\boldsymbol{f}(\textbf{A}) \textbf{L}  \sim \textbf{L}$ for short), the better the graph topology matches the graph tasks, which 
we provide a motivating example with mild assumptions to validate in the next section. 
This can be understood through various scenarios: 
\begin{itemize}

\item In graph filtering, the fact that $\boldsymbol{f}(\textbf{A}) \textbf{L}  \sim \textbf{L}$ indicates that signals in $\textbf{L}$, whether low or high frequency,  are well preserved after being filtered by $\boldsymbol{f}(\textbf{A})$. This preservation of information suggests a good match between task labels and the graph topology.

\item Considering the concept of the homophily/heterophily, when $\boldsymbol{f}(\textbf{A}) \textbf{L}  \sim \textbf{L}$ means linked nodes described by $\boldsymbol{f}(\textbf{A})$ likely belong to the same class or have similar characteristics, 
which is aligned with the homophily principle.

\item In terms of Dirichlet energy, if $\boldsymbol{f}(\textbf{A}) \textbf{L}  \sim \textbf{L}$ means low Dirichlet energy of $\textbf{L}$ w.r.t. the (augmented) normalized Laplacian derived from $\boldsymbol{f}(\textbf{A})$. 
This indicates the smoothness of $\textbf{L}$ on  $\boldsymbol{f}(\textbf{A})$, suggesting a good match between the graph topology and tasks.

\end{itemize}

All above, we define $\mathcal{I}(\textbf{A})$ to quantify the degree of compatibility between graph topology and graph tasks, 
$$ \mathcal{I}(\textbf{A}) = \mathcal{S} \left( \textbf{L} \| \boldsymbol{f}(\textbf{A}) \textbf{L}  \right), $$
where $\mathcal{S}(\cdot)$ is a similarity measurement, which can be chosen according to the requirements of the problem, for example, similarity induced by the inner product or Euclidean distance. 

To find a graph topology that matches the graph tasks well, our goal is to maximize $\mathcal{I}(\textbf{A})$ over $\textbf{A}$.
However, $\mathcal{I}(\textbf{A})$ would be trivially maximized with $\boldsymbol{f}(\textbf{A}) = \textbf{I} $, where the graph topology is completely removed. 
To address this issue, we introduce a regularization function $ \mathcal{R}(\cdot)$, which penalizes $\mathcal{I}(\textbf{A})$ if $\boldsymbol{f}(\textbf{A})$ is close to the identity matrix, to minimize the modification of $\boldsymbol{f}(\textbf{A})$. 
We also demonstrate such regularization from the perspective of graph denoising in the next section. 
Therefore, the optimization problem is formulated as
\begin{equation}
    \max_{\textbf{A}} \ \mathcal{C}(\textbf{A}) = \mathcal{I}(\textbf{A}) - \lambda \mathcal{R} \left( \textbf{A} \right)
\end{equation}
where $\lambda$ is a hyperparameter used to balance the trade-off between $\mathcal{I}(\textbf{A})$ and $\mathcal{R} (\textbf{A})$.
However, there are $2^{n\times n}$ possible states of $ \textbf{A} $,
making the optimization of $\mathcal{C}(\textbf{A})$ computationally intractable due to its nondifferentiability and combinatorial nature. 
Moreover, the optimization does not consider the original topology. 
Therefore, instead of optimizing the problem globally, 
a more practical and effective approach is to start with the originally observed graph and locally optimize it, as we will present in the next subsection. 

Now we zoom in the general compatibility to the node level:
for every node $v_i$, we use the node level similarity function $\mathcal{S}_{  \textbf{v}}(\cdot)$
to measure the similarity of its one-hot label vector $\textbf{L}_{i, :} $ and its normalized soft label vector $ \overline{\textbf{L}}_{i, :}$ filtered by $\boldsymbol{f}(\textbf{A})$, and node level regularization function $\mathcal{R}_{  \textbf{v}}(\cdot)$ to measure the regularization of its connectivity to other nodes, specifically, the $i^{\text{th}}$ row of $\boldsymbol{f}(\textbf{A})$. 
Let $\mathcal{C}_{\textbf{A}}(v_i)$ denote the compatibility between the topology of node $v_i$' and graph tasks. 
The overall compatibility metric becomes the following 
\begin{equation}
\label{eq:I_A_node_wise}
    \mathcal{C}(\textbf{A}) = 
    \sum_{v_i \in \mathcal{V}_{t}} \mathcal{C}_{\textbf{A}}(v_i) = 
    \sum_{v_i \in \mathcal{V}_{t}} \mathcal{I}_{\textbf{A}}(v_i) - \lambda \mathcal{R}_{\textbf{A}}(v_i),
\end{equation}
where $ \mathcal{I}_{\textbf{A}}(v_i) = \mathcal{S}_{\textbf{v}} ( \textbf{L}_{i, :} \| \overline{\textbf{L}}_{i, :} )$ and $\mathcal{V}_{t} \subseteq \mathcal{V}$ is the node set. 
Notice that $\mathcal{V}_{t}$ can be not only the whole node set but also be chosen flexibly as the node set whose compatibility we care about most in correspondence to the circumstances.   
For example, in adversarial attack tasks, $\mathcal{V}_{t}$ can be a collection of targeted nodes to be attacked; in graph node classification tasks, $\mathcal{V}_{t}$ can be the nodes to be predicted.
In this work, we choose the inner product as the similarity metric, where $\mathcal{I}_{\textbf{A}}(v_i)$ can be evaluated as $\mathcal{I}_{\textbf{A}}(v_i) = \overline{\textbf{L}}_{i, c_i}$, and $c_i$ is the true label of $v_i$, and the reciprocal of degree as the regularizer, where $\mathcal{R}_{\textbf{A}}(v_i)$ can be evaluated as $\mathcal{R}_{\textbf{A}}(v_i) = 1 / d_i$, and $d_i$ is the degree of node $v_i$.

\subsection{TopoInf: measuring the influence of graph topology on tasks }

From the above analysis, we can evaluate how well the overall graph topology and learning tasks are in accordance with $\mathcal{C}(A)$.
We are interested in optimizing $\mathcal{C}(A)$, which can improve the compatibility between graph topology and tasks, and, by extension, improve the performance of the learning model. 
We are also interested in characterizing the influence of modifying part of topology on the degree of task compatibility, which can provide more interpretability of graph learning and meanwhile guide the topology modification for better task performance. To maximize $\mathcal{C}(\textbf{A})$, we can
take the ``derivative'' of $\mathcal{C}(\textbf{A})$ by obtaining the change of $\mathcal{C}(\textbf{A})$ with local topology perturbation. 
Here we focus on the topology of edges and the same analysis can be simply extended to nodes and substructures.  
The influence of edge $e_{ij}$ can be measured by the difference between $\mathcal{C}(\textbf{A})$ on original (normalized) adjacency matrix $\textbf{A}$ and $\mathcal{C}(\textbf{A}^{\prime})$ on the modified (normalized) adjacency matrix $\textbf{A}^{\prime}$ obtained by removing edge $e_{ij}$. 
Therefore, we define such topology influence as TopoInf of edge $e_{ij}$, denoted as $\nabla \mathcal{C}_{\textbf{A}}(e_{ij})$, is given by
\begin{equation}
    \nabla \mathcal{C}_{\textbf{A}}(e_{ij}) = \mathcal{C} ( \textbf{A}^{\prime} ) - \mathcal{C} ( \textbf{A} ) = \sum_{v_k \in \mathcal{V}} \mathcal{C}_{\textbf{A}^{\prime}}(v_k) - \mathcal{C}_{\textbf{A}}(v_k). 
\end{equation}
The sign of \textbf{TopoInf} reflects the positive or negative influence of removing the edge on the matching of topology and tasks, which also indicates that removing the edge can increase or decrease the model performance. 
Remark that edges with positive TopoInf correspond to edges with a ``negative'' effect on the model performance, such that removing those can bring a positive influence, and vice versa. 
The absolute value of \textbf{TopoInf} measures the magnitude of influence: a higher absolute value means a higher influence. 

It should be noted that to compute \textbf{TopoInf} of an edge, we do not need to recompute $\boldsymbol{f}(\textbf{A}^{\prime}) \textbf{L} $, which is computationally expensive. 
The difference between $\mathcal{I} ( \textbf{A}^{\prime} )$ and $ \mathcal{I} ( \textbf{A} ) $ after removing $e_{ij}$ is limited within $K$ hop neighborhood of $v_i$ and $v_j$ considering a $K$-order filter $\boldsymbol{f}(\textbf{A})$. 
Therefore, we only need to compute the difference of node influence on the neighborhood affected by the edge removal in Equation \ref{eq:I_A_node_wise}.

Remark that we do not assume that we have all the true labels.
Using all labels is for the convenience of analysis to demonstrate the effectiveness of the metric. 
In practical graph tasks where not all labels are available, 
we can use pseudo labels obtained by MLP or GNN models as replacements for true labels and compute estimated TopoInf based on pseudo labels.
This will introduce errors, but can still be effective which we will show in our experiment. 

Moreover, due to the presence of nonlinearity in GNN models, the extraction of $\boldsymbol{f}(\textbf{A})$ can be challenging for graph models, and an approximation of $\boldsymbol{f}(\textbf{A})$ is required. 
We assume that $\boldsymbol{f}(\textbf{A})$ can be obtained or approximated for the graph models, and a better approximation can improve the accuracy of the influence metric. 
Here, we present our approach to obtaining, approximating and learning $\boldsymbol{f}(\textbf{A})$ for several representative GNN models.
\begin{itemize}
\item \textbf{Obtained graph filters.} For SGC \cite{SGC}, S2GC  \cite{S2GC}, APPNP \cite{APPNP}, and other decoupled GNNs,  
 $\boldsymbol{f}(\textbf{A})$ can be directly obtained by Equation \ref{eq:decoupled_gnn}.

\item \textbf{Approximated graph filters.} For GCN \cite{GCN} and other GCN-like models, 
the presence of nonlinearities introduces additional complexities to $\boldsymbol{f}(\textbf{A})$, as it becomes dependent not only on the observed graph topology but also on the learnable parameters and the activation function. 
In this study, we adopt an approximate approach to estimate $\boldsymbol{f}(\textbf{A})$ by removing the nonlinear activation to simplify the analysis, as works in \cite{SGC,IsHomophilyaNecessityforGNN}. 
For GCNII \cite{GCNII} and other GNNs with residual connections between layers, 
we can employ a similar approximate approach as in GCN to estimate $\boldsymbol{f}(\textbf{A})$ by removing the nonlinear normalization. 
This approximation without nonlinearity, can still capture the effects of graph learning models performing on the topology, as we will show in our experiments.  

\item \textbf{Learned graph filters.} For GPRGNN \cite{GPRGNN}, ChebNet \cite{ChebNet}, and other GNNs driven by learning filters,
the filter weights are learned based on backpropagation, 
where we are not able to obtain the graph filter $\boldsymbol{f}(\textbf{A})$ before training. 
In this study, we adopt two approximate approaches to deal with this problem. 
One approach is to train the filter-based GNN model, obtain the trained filter weights, and then apply the learned graph filter as $\boldsymbol{f}(\textbf{A})$. 
Another approach is to use fixed filter weights or predefined filters as the graph filter $\boldsymbol{f}(\textbf{A})$, such as $\hat{\textbf{A}}^K$.
This approach ignores the effect of the learned filter weights in the GNN model but can still preserve the order information and other prior knowledge from the GNN model. 
Moreover, for GAT \cite{GAT} and other attention-based GNNs,  unlike GNNs driven by learning filters, 
the learned filters are implicitly determined by the weights of the neural network and the input features,
and the learned filters may change 
due to the randomness during initialization and training. 
In this case, we apply only the second approach with a predefined fixed filter as an approximation.  
\end{itemize}

\begin{table}[h]
\vspace{-12mm}
    \centering
    \caption{\textbf{Approximation approach of $\boldsymbol{f}(\textbf{A})$ for representative GNN models.} 
     Here we take K-layer GNN models as examples.
     $\alpha$ is predefined hyper-parameter and $\gamma$ is learnable parameter. 
     $\mathbf{H}^{(0)} = \boldsymbol{g}_{\theta}(\textbf{X})$.
    }
\label{tab:filter_approximation}
\vspace{1mm}

    \resizebox{1\textwidth}{!}
{
\begin{tabular}{llll}
\toprule[1.5pt]
Model & Neighbor Aggregation & Filter & Type\\
\midrule
SGC 
& $ \hat{\textbf{L}} = \text{softmax}(\hat{\textbf{A}}^K \textbf{X} \textbf{W} ) $ 
& $\boldsymbol{f}(\textbf{A}) = \hat{\textbf{A}}^K$ 
& obtained  
\\

S2GC 
& $ \hat{\textbf{L}}=\operatorname{softmax}\left(\frac{1}{K} \sum_{k=1}^K\left((1-\alpha) \hat{\textbf{A}}^k +\alpha I \right) \mathbf{X} \textbf{W}\right) $ 
& $\boldsymbol{f}(\textbf{A}) =\frac{1}{K} \sum_{k=1}^K (1-\alpha) \hat{\textbf{A}}^k +\alpha I$ 
& obtained  
\\

APPNP 
& $\mathbf{H}^{(\ell+1)}=\left(1-\alpha\right) \hat{\mathbf{A}} \mathbf{H}^{(\ell)}+\alpha \mathbf{H}^{(0)}$ 
& $\boldsymbol{f}(\textbf{A}) = (1-\alpha)\hat{\textbf{A}}^K + \sum_{k=0}^{K-1}\alpha (1-\alpha)^k\hat{\textbf{A}}^k $
& obtained 
\\

GCN 
& $ \mathbf{H}^{(\ell+1)}=\sigma\left( \hat{ \textbf{A} } \mathbf{H}^{(\ell)} 
\mathbf{W}^{(\ell)}  \right) $ 
& $\boldsymbol{f}(\textbf{A}) = \hat{\textbf{A}}^K$ 
& approximated  
\\

GCNII 
& $\mathbf{H}^{(\ell+1)}=\sigma\left(\left(\left(1-\alpha\right) \hat{\mathbf{A}} \mathbf{H}^{(\ell)}+\alpha \mathbf{H}^{(0)}\right)\mathbf{W}^{(\ell)}\right)$ 
& $\boldsymbol{f}(\textbf{A}) = (1-\alpha)\hat{\textbf{A}}^K + \sum_{k=0}^{K-1}\alpha (1-\alpha)^k\hat{\textbf{A}}^k $
& approximated
\\

GPRGNN 
& $ \hat{\textbf{L}}=\operatorname{softmax} \left( \sum_{k=0}^K 
\gamma_k \hat{\textbf{A}}^k  \mathbf{H}^{(0)} \right) $ 
& $\boldsymbol{f}(\textbf{A}) = \sum_{k=0}^K 
\gamma_k \hat{\textbf{A}}^k $ 
& learned
\\
\bottomrule[1.5pt]
\end{tabular}
}
\vspace{-10mm}
\end{table}

\section{Motivation}

In this section, we provide a theoretical analysis and motivating example to validate our methodology and increase its comprehensibility. Specifically, we provide alternative perspectives to explain the meaning and validate the necessity of $\mathcal{I}( \textbf{A} )$ and $\mathcal{R}( \textbf{A} )$ in $\mathcal{C}( \textbf{A} )$.

We use contextual stochastic block model (cSBM) \cite{CSBM} for analysis, which is a widely used model for complex graph structures that integrates contextual features and graph topology \cite{ASGC,ESNR&GPS,CSBM}. 
In cSBM, node features are generated by the spiked covariance model, which follows a Gaussian distribution, the mean of which depends on the community assignment. 
The underlying assumption is that the node features are ``informative'' and can be perceived as label embedding vectors with Gaussian noise. 
The graph topology is generated by SBMs, resulting in communities of nodes connected by certain edge densities. 

\begin{definition}
\label{def:csbm}
\vspace{-1mm}
\textbf{(Contextual Stochastic Block Model, cSBM)}.  
 $\mathcal{G}$ has $n$ nodes belonging to $c$ communities, with intra/inter-community edge probabilities of $p$ and $q$. 
The feature matrix is $ \textbf{X} = \textbf{F} + \textbf{N} $, 
where $ \textbf{N}_{ij} \sim \mathcal{N}(0; \sigma ^2)$ are i.i.d. Gaussian noise,  
$ \textbf{F}_{i, :} = \boldsymbol{\mu_{c_i}^{\top}} $ is the $d$ dimensional feature vector for the center of the community $c_i$, which is the community $v_i$ belongs to.   
Therefore, $ \textbf{F} = \textbf{L} \boldsymbol{\mu} $, where $\textbf{L}$ is the matrix of one hot label that denotes the community to which each node belongs, and $\boldsymbol{\mu} = \left( \boldsymbol{\mu}_{1}, \boldsymbol{\mu}_{2}, \dots, \boldsymbol{\mu}_{c} \right) ^{\top} \in \mathbb{R}^{c\times d}$.
\vspace{-2mm}
\end{definition}

\subsection{Theoretical Analysis}

Here we present a theoretical analysis to validate our TopoInf and enhance its comprehensibility.
According to the definition of cSBMs \cite{CSBM},
the feature matrix $ \textbf{F} = \textbf{L} \mu $ is directly related to the prediction $\textbf{L}$ and contains sufficient information for the graph learning task, which is also discussed and empirically verified in \cite{AllWeHaveIsLowPassFilters}. 
In this part, we follow the setting in \cite{AllWeHaveIsLowPassFilters} as a motivating example, and 
treat $ \textbf{F}$ as the true signals and  $\boldsymbol{f}(\textbf{A}) \textbf{X}$ as the prediction, and 
evaluate the effect of the low-pass filter by measuring the difference between the filtered feature matrix $\boldsymbol{f}(\textbf{A}) \textbf{X}$ and the true feature matrix $\textbf{F}$. 
To analyze the difference, we focus on the low-pass filter, which is true for most GCN-like models \cite{GCN,SGC,APPNP,AllWeHaveIsLowPassFilters},
use $\left\| \boldsymbol{f}(\textbf{A}) \textbf{X} - \textbf{F} \right\|$ as the prediction error introduced by the learning model and the topology, and consequently establish the following theorem. 

 \begin{theorem}
\label{thm:f_A_effect}
For $ 0 < \delta < 1 $, with probability at least $ 1 - \delta $, we have
\begin{equation}
\label{eq:f_A_effect}
    \left\| \boldsymbol{f}(\textbf{A}) \textbf{X} - \textbf{F} \right\| \leq 
    c_1  \left\| \boldsymbol{f}(\textbf{A}) \textbf{L} - \textbf{L} \right\| + 
    c_2  \left\| \boldsymbol{f}(\textbf{A}) \right\|,
\end{equation}
where $c_1 = O\left( \left\| \boldsymbol{\mu} \right\| \right)$, 
$c_2 = O\left( \mathbb{E} 
\left\{ \left\| \textbf{N} \right\| \right\} / \delta \right)$.
\end{theorem}

\begin{proof}
By substituting $\textbf{X} = \textbf{F} + \textbf{N} $, we obtain
$
\left\| \boldsymbol{f}(\textbf{A}) \textbf{X} - \textbf{F} \right\| \leq 
    \left\| \boldsymbol{f}(\textbf{A}) \textbf{F} - \textbf{F} \right\| + 
    \left\| \boldsymbol{f}(\textbf{A}) \textbf{N} \right\|.
$
For the first term, by substituting $\textbf{F} = \textbf{L} \mu $, we obtain
$
    \left\| \boldsymbol{f}(\textbf{A}) \textbf{F} - \textbf{F} \right\| \le
    \left\| \boldsymbol{f}(\textbf{A}) \textbf{L} - \textbf{L} \right\| 
    O\left( \left\| \boldsymbol{\mu} \right\| \right).
$
For the second term, by Markov inequality, we obtain, $\forall t > 0$, 
$
\text{Prob} \left\{ \left\| \boldsymbol{f}(\textbf{A}) \textbf{N} \right\| > t \right\} \le \frac{ \mathbb{E} ( \left\| \boldsymbol{f}(\textbf{A}) \textbf{N} \right\| ) }{t}.
$
By substituting $ t = \frac{ \mathbb{E} ( \left\| \boldsymbol{f}(\textbf{A}) \textbf{N} \right\| ) }{\delta} $, we obtain 
$
\text{Prob} \left\{ \left\| \boldsymbol{f}(\textbf{A}) \textbf{N} \right\| \le 
O\left( \mathbb{E} 
\frac
{\left\{ \left\| \textbf{N} \right\| \right\} }
{\delta}
\right)
\left\| \boldsymbol{f}(\textbf{A}) \right\| \right\}
\le 
1 - \delta. 
$

\end{proof}

Theorem \ref{thm:f_A_effect} suggests that the effect of a graph filter is upper bounded by two parts, as shown on the right side of Equation \ref{eq:f_A_effect}. 
For the first part, $\left\| \boldsymbol{f}(\textbf{A}) \textbf{L} - \textbf{L} \right\|$ represents the bias introduced by the filter  $\boldsymbol{f}(\textbf{A})$  
on the ideal feature matrix $ \textbf{F} = \textbf{L} \mu $, which is related to the label matrix $ \textbf{L} $, weighted by a constant of the embedding vectors of the labels. 
For the second part, $\left\| \boldsymbol{f}(\textbf{A}) \right\|$ weighted by the constant related to the noise, reflects the effect that
$ \boldsymbol{f}(\textbf{A}) $ filters the noise $\textbf{N}$. 
Intuitively, better performance can be achieved when bias is minimized and more noise is filtered out.
Our definition of $\mathcal{C}(\textbf{A})$, which captures the compatibility between graph topology and tasks, aligns with Theorem \ref{thm:f_A_effect}, as it considers both parts. 
On the one hand, $\mathcal{I}(\textbf{A})$ corresponds to $\left\| \boldsymbol{f}(\textbf{A}) \textbf{L} - \textbf{L} \right\|$, as a smaller value of $\left\| \boldsymbol{f}(\textbf{A}) \textbf{L} - \textbf{L} \right\|$ indicates a higher similarity between $\boldsymbol{f}(\textbf{A}) \textbf{L}$ and $\textbf{L}$.
On the other hand, $\mathcal{R}(\textbf{A})$ aligns with $\left\| \boldsymbol{f}(\textbf{A}) \right\|$, as a smaller value of $\left\| \boldsymbol{f}(\textbf{A}) \right\|$ implies less regularization on $\boldsymbol{f}(\textbf{A})$.
Furthermore, our definition of $\mathcal{C}(\textbf{A})$ offers more flexibility, as the similarity function $\mathcal{S}(\cdot)$ is not limited to Euclidean distance and the regularization function $\mathcal{R}(\cdot)$ is not limited to L2-norm.
This flexibility allows for a broader range of choices in defining the compatibility measure based on specific requirements and characteristics of the graph learning task at hand.

We further investigate the bias introduced by $\boldsymbol{f}(\textbf{A})$ on $ \textbf{F} $ by 
analyzing the change in the distance between a node and its farthest inter-class node, which represents the radius of nodes belonging to a different community distributed, centered on the node. 
A larger radius indicates easier classification. 
For noise filtering, we measure the change in variance after applying the filter. 
In the following, we present the results in detail.

\begin{theorem}
\label{rmk:f_A_bias_denoise}
Suppose that $ \boldsymbol{f}(\textbf{A}) = \text{RowNorm}(\sum_{k=0}^{K} \gamma_k \hat{\textbf{A}}^k) $, $\sum_{k=0}^{K} \gamma_k = 1$ and $\gamma_k >= 0$, 
which are low-pass filters \cite{GPRGNN,BernNet}.
We define the distance between node $v_i$ and the farthest node to $v_i$ that belongs to a different community 
as
$ \mathcal{D}(\textbf{X}, v_i) = \max_{v_j \in \mathcal{V}, c_i \neq c_j} \left\| \textbf{X}_{i, :} - \textbf{X}_{j, :} \right\| $. 
We have 
\begin{itemize}
    \item The maximum distance between $v_i$ and the farthest node to $v_i$ that belongs to a different community decreases after applying the graph filter:
    $\forall v_i \in \mathcal{V}, \mathcal{D}(\textbf{F}, v_i) \ge \mathcal{D}( \boldsymbol{f}(\textbf{A}) \textbf{F}, v_i )$.
    \item The total variance of the noise decreases: $ \left\| \text{Var} \{ \textbf{N} \}  \right\|_1 \ge \left\| \text{Var} \{ \boldsymbol{f}(\textbf{A}) \textbf{N} \} \right\|_1 $, where $\left\| X \right\|_1 = \sum_{i} \sum_{j} |X_{i,j}|$.
\end{itemize}
\end{theorem}

\begin{proof}
For the first result, 
by definition, 
\begin{footnotesize}
\begin{align*}
\mathcal{D}(\boldsymbol{f}(\textbf{A}) \textbf{F}, v_i) 
& = \max_{v_j \in \mathcal{V}, c_i \neq c_j}
\left\| 
\sum_{v_k \in \mathcal{V}} \left( \boldsymbol{f}(\textbf{A}) \right)_{i, k} \textbf{F}_{k, :} - \sum_{v_k \in \mathcal{V}} \left( \boldsymbol{f}(\textbf{A}) \right)_{j, k} \textbf{F}_{k, :} \right\| \\
& = \max_{v_j \in \mathcal{V}, c_i \neq c_j}
\left\| 
\left( \boldsymbol{f}(\textbf{A}) \right)_{i, :} \textbf{L} \mu - \left( \boldsymbol{f}(\textbf{A}) \right)_{j, :} \textbf{L} \mu \right\| 
= \max_{v_j \in \mathcal{V}, c_i \neq c_j} 
\left\| 
\sum_{c = 1}^{C} l_{c}^{(i, j)} \mu_{c}
\right\|, 
\end{align*}
\end{footnotesize}
where $l^{(i, j)} = \left( \boldsymbol{f}(\textbf{A}) \right)_{i, :} \textbf{L} - \left( \boldsymbol{f}(\textbf{A}) \right)_{j, :} \textbf{L}$ and $ l^{(i, j)} \in \mathbb{R}^{1\times C} $, and $\left( \boldsymbol{f}(\textbf{A}) \right)_{i, :}$, $\left( \boldsymbol{f}(\textbf{A}) \right)_{j, :}$ represents the $i^{\text{th}}$ and $j^{\text{th}}$ row of $\boldsymbol{f}(\textbf{A})$ respectively. 
And $\exists w_{m, n}^{(i, j)}, \sum_{c = 1}^{C} l_{c}^{(i, j)} \mu_{c} = \sum_{m = 1}^{C} \sum_{n = 1}^{C} w_{m, n}^{(i, j)} (\mu_{m} - \mu_{n})$, such that $\forall 1 \le c, m, n, \le C, w_{c, c}^{(i, j)} = 0, w_{c, m}^{(i, j)} w_{c, n}^{(i, j)} \ge 0, \min(w_{m, n}^{(i, j)}, w_{n, m}^{(i, j)}) = 0$.
This can be seen as the process that the positive  
elements in 
$l^{(i, j)}$ 
distribute its value to negative elements in 
$l^{(i, j)}$.
Then we can obtain
$ \sum_{c = 1}^{C} | l_{c}^{(i, j)} | = 2 \sum_{m = 1}^{C} \sum_{n = 1}^{C} |w_{m, n}^{(i, j)}| $.
Due to the fact that all elements in $\boldsymbol{f}(\textbf{A})$ are non-negative and each row sum of $\boldsymbol{f}(\textbf{A})$ equals one, we obtain $\forall v_i, v_j \in  \mathcal{V}$, $\sum_{c = 1}^{C} l_{c}^{(i, j)} = 0$ and $\sum_{c = 1}^{C} | l_{c}^{(i, j)} | <= 2$.
Further, we denote $v_{j^{*}} = \text{argmax}_{v_j \in \mathcal{V}, c_i \neq c_j} \left\| \textbf{F}_{i, :} - \textbf{F}_{j, :} \right\|$ for $\forall v_i \in \mathcal{V}$, then we obtain 
$
\mathcal{D}(\textbf{F}, v_i) = \left\| \textbf{F}_{i, :} - \textbf{F}_{j^{*}, :} \right\| =  \left\| \mu_{c_i} - \mu_{c_{j^{*}}} \right\|. 
$
Subsequently, we derive
\begin{footnotesize} 
\begin{align*}
\mathcal{D}(\boldsymbol{f}(\textbf{A}) \textbf{F}, v_i) 
& = \max_{v_j \in \mathcal{V}, c_i \neq c_j} 
\left\| 
\sum_{m = 1}^{C} \sum_{n = 1}^{C} w_{m, n}^{(i, j)} (\mu_{m} - \mu_{n})
\right\| 
\le 
\max_{v_j \in \mathcal{V}, c_i \neq c_j} 
\sum_{m = 1}^{C} \sum_{n = 1}^{C} |w_{m, n}^{(i, j)}| \left\| 
(\mu_{m} - \mu_{n})
\right\| \\
& \le 
\max_{v_j \in \mathcal{V}, c_i \neq c_j} 
\sum_{m = 1}^{C} \sum_{n = 1}^{C} |w_{m, n}^{(i, j)}| \left\| \mu_{c_i} - \mu_{c_{j^{*}}} \right\| 
\le 
\left\| \mu_{c_i} - \mu_{c_{j^{*}}} \right\| = \mathcal{D}(\textbf{F}, v_i).  
\end{align*}
\end{footnotesize}

For the second result, $\boldsymbol{f}(\textbf{A})$ is a row stochastic matrix, so absolute values of all eigenvalues of $\boldsymbol{f}(\textbf{A})$ are within 1. Therefore, $\left\| \text{Var} \{ \boldsymbol{f}(\textbf{A}) \textbf{N} \} \right\|_1 = \textbf{trace} \left( \mathbb{E} ( 
   \boldsymbol{f}(\textbf{A}) \textbf{N} \textbf{N}^{\top}   \boldsymbol{f}(\textbf{A})^{\top} ) \right) 
   \le n \sigma^2 \cdot \textbf{trace} \left( \mathbb{E}(\boldsymbol{f}(\textbf{A}) \boldsymbol{f}(\textbf{A})^{\top} ) \right) \le n^2 \sigma^2 = \left\| \text{Var} \{ \textbf{N} \}  \right\|_1 $.

\end{proof}

Theorem \ref{rmk:f_A_bias_denoise} shows that the distances analyzed decrease after applying the graph filter, indicating that the difficulty increases for classification and reflects that bias is introduced. 
For noise, the variance of the noise decreases through the low-pass filter,
suggesting that the noise is reduced.
In order to analyze the compatibility comprehensively, we need to consider these two effects simultaneously.
This also explains why $\mathcal{I}(\textbf{A})$ and $\mathcal{R}(\textbf{A})$ are necessary in $\mathcal{C}(\textbf{A})$.

\subsection{Motivating Example}

We further present a case study on cSBM to analyze the performance of GCN with different feature noise and
show the effects of $\boldsymbol{f}(\textbf{A})$.

\begin{figure}[t]
    \centering
    \includegraphics[width=\columnwidth]{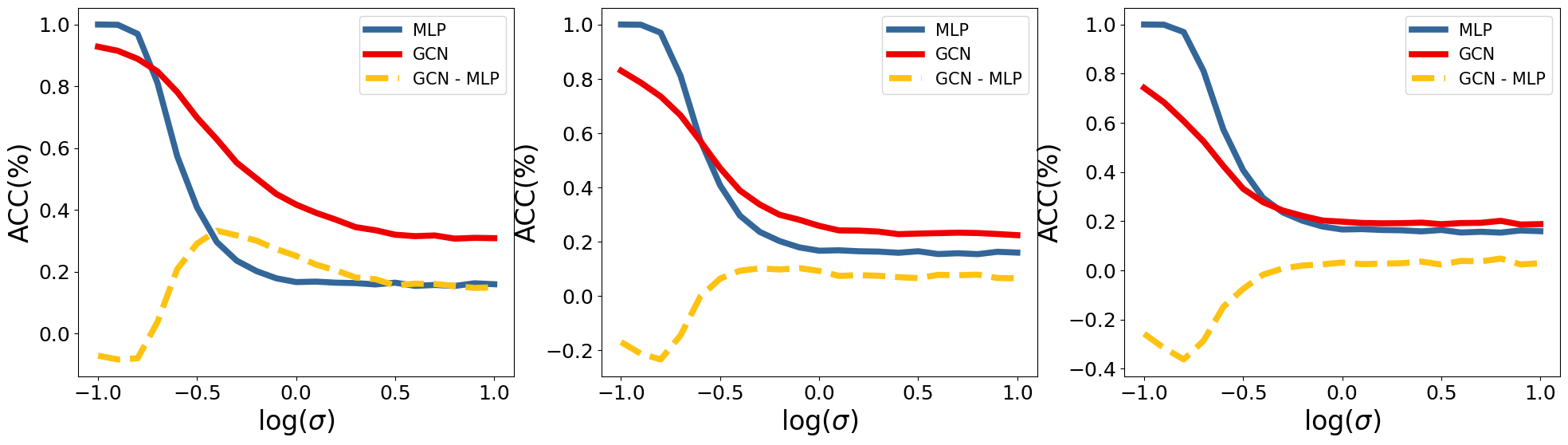}
    \vspace{-6mm}
    \caption{
    \textbf{The performance of GCN and MLP on the synthetic graph datasets generated by cSBM} with different $\sigma, p, q$, corresponding to the noise variance of features, intra-community and inter-community connection probability. 
    The synthetic datasets share the same statistics with Cora \cite{CoraAndCiteSeer}, such as the number of nodes and edges and the dimension of features. 
    $p$ and $q$ used to generate SBMs are $(0.9, 0.1)$, $(0.8, 0.2)$, and $(0.7, 0.3)$ respectively from left to right.
    }
    \label{fig:motivating_example}
\vspace{-6mm}    
\end{figure}

As shown in Figure \ref{fig:motivating_example}, MLP outperforms GCN when the noise variance $\sigma$ is close to zero. The performance gap occurs because $ \boldsymbol{f}(\textbf{A}) $ biases the prediction of GCN and the lower the quality of the topology, i.e., the lower $p$, the larger the gap, from left to right. 
However, as $\sigma$ increases, the performance of both MLP and GCN decreases, but the performance of GCN decreases more slowly than that of MLP because GCN smooths the noise through topology. 
During this phase, the performance of MLP gradually approaches that of GCN until GCN outperforms MLP. 
When the noise variance is too large to dominate the feature, the performance of MLP and GCN drops and converges. The convergence performance of GCN is usually higher than MLP, as MLP degenerates into random guesses, while GCN is still able to learn from topology. 

This motivating experiment shows that, compared with MLPs which do not utilize graph topology, 
the bias introduced by the low-pass filter $ \boldsymbol{f}(\textbf{A}) $ results in that MLPs outperform GNNs when the noise variance in features is near zero, suggesting that this effect is generally bad for performance.
However, compared to MLPs, low-pass filters $ \boldsymbol{f}(\textbf{A}) $ also provide an additional noise reduction ability, which improves the performance of GNNs when the noise variance in features is greater than zero.
Therefore, to fully analyze the compatibility between topology and task in GNN models, 
we design both $\mathcal{I}(\textbf{A})$ and $\mathcal{R}(\textbf{A})$ in $\mathcal{C}(\textbf{A})$ to analyze and quantify these two effects simultaneously.

\vspace{-1.5mm}
\section{Experiments}
\vspace{-1mm}

\label{sec:experiments}

We conduct experiments to validate the performance of TopoInf: 
\begin{itemize}
\vspace{-1.5mm}

\item validate the effectiveness of TopoInf on graphs with ground truth labels; 
\item estimate TopoInf based on pseudo labels and utilize the estimated TopoInf to refine graph topology for improving model performance; 
\item show further applications of model improvement, by modifying topology guided by TopoInf, with DropEdge \cite{DropEdge} as an example.

\vspace{-1.5mm}
\end{itemize}

\noindent\textbf{Setup.} 
We choose six real-world graph datasets, namely Cora, CiteSeer, PubMed, Computers, Photos and Actor in our experiments \cite{CoraAndCiteSeer,Pubmed,AmazonComputerPhoto,GeomGCNTexasAndCornell}.
For Computers and Photos, we randomly select 20 nodes from each class for training, 30 nodes from each class for validation, and use the remaining nodes for testing.
For other datasets, we use their public train/val/test splits.
We choose nine widely adopted models, namely GCN, SGC, APPNP, GAT, GIN, TAGCN, GPRGNN, BernNet and GCNII to work on \cite{GCN,SGC,APPNP,GAT,GIN,TAGCN,GPRGNN,BernNet,GCNII}.

\subsection{Validating TopoInf on Graphs}

\noindent\textbf{Can TopoInf reflect edges' influence on model performance?} 
We first conduct experiments on TopoInf computed using ground truth labels.
Based on the former analysis, the sign of TopoInf implicates the directional influence of removing $e_{ij}$ on the model performance, while the absolute value of TopoInf reflects the magnitude of the influence.
We validate these implications in real datasets.
For each dataset, 
we compute TopoInf based on ground truth labels and partition edges with positive TopoInf into the positive set and edges with negative TopoInf into the negative set.
Then we sequentially remove edges from the positive/negative set ordered by the absolute values of their TopoInf in a descendant manner and record the performance of models.

As is shown in Figure \ref{fig:topoinf_guided_model_perf},
performance curves exhibit an S-shaped pattern.  
Specifically, after removing edges in the positive set, the model performance increases, while removing edges in the negative set diminishes the performance. 
Moreover, removing edges with higher absolute TopoInf values results in a greater derivative of the performance plots.
These are consistent with our previous analysis and show the effectiveness of TopoInf.

\begin{figure}[thbp]
  \centering
  \begin{minipage}[b]{\textwidth}
    \centering
    \includegraphics[width=\textwidth]{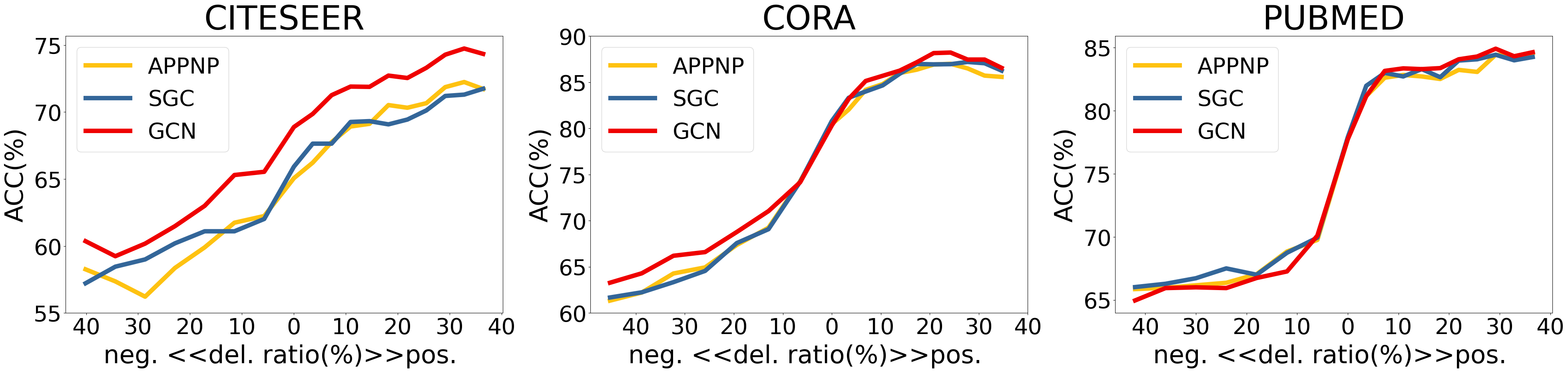}
    \includegraphics[width=\textwidth]{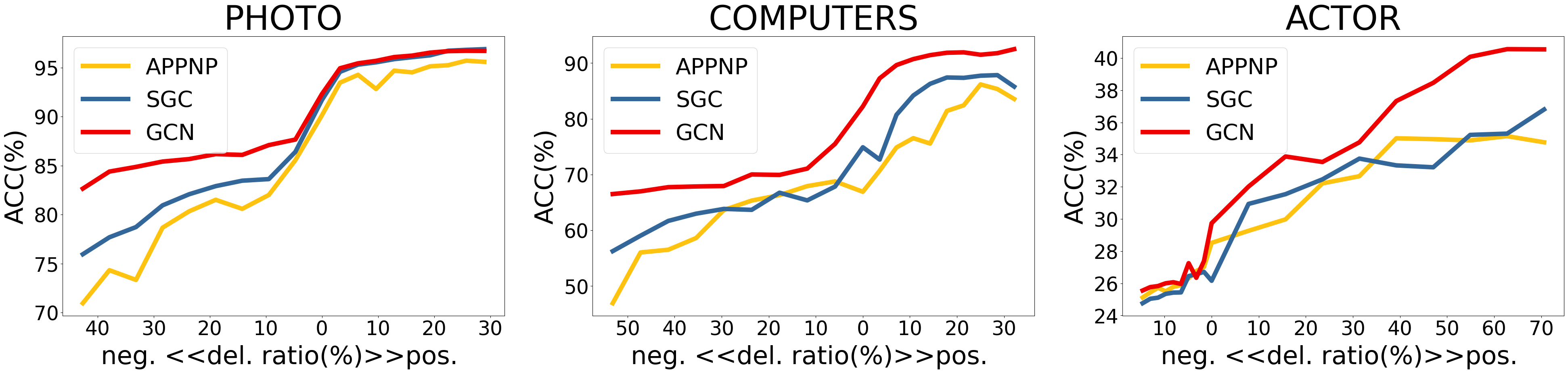}
  \end{minipage}
  \vspace{-6mm}
  \caption{
    \textbf{Performance change while deleting edges by TopoInf.} Horizontal axis's meaning: zero point corresponds to the original graph without deleting any edges. 
    The absolute value of the coordinate denotes the ratio of deleted edges to all edges. 
    The negative/positive coordinate corresponds to the case that the graph is obtained by deleting edges with negative/positive TopoInf values of the corresponding proportion, in the descendant order of the absolute value of TopoInf. 
    The vertical axis is the accuracy of the node classification in the GNN models. 
    The TopoInf values are calculated with $\mathcal{V}_t$ in Equation \ref{eq:I_A_node_wise} set as the test set.
    For each step, we remove 10\% of edges in the negative/positive set.
}
\label{fig:topoinf_guided_model_perf}
\vspace{-6mm}
\end{figure}

\begin{sidewaystable} [!htb]
\renewcommand\arraystretch{1.4} 
\normalsize
\centering
\tabcolsep=0.15cm   

{
\caption{\small{\textbf{Performance after deleting edges based on estimated TopoInf on different models and datasets.} 
The estimated TopoInf values are calculated with $\mathcal{V}_t$ in Equation \ref{eq:I_A_node_wise} set as the whole nodes.
The ratio of the edges removed to all edges is a hyperparameter chosen from $\{ 0.02, 0.04, 0.06, 0.08 \}$ and is determined by validation.
We run experiments 10 times and report testing accuracy with a 95\% confidence interval.
}}
\label{tab:exp3_compare_model}
\resizebox{.95\columnwidth}{!}{
\begin{tabular}{l l c c c c c c c c c}
    \toprule[1.5pt]
    Datasets & Method & GCN & SGC & APPNP & GAT & GIN & TAGCN & GPRGNN & BERNNET & GCNII \\
    \midrule
    \multirow{3}*{CORA} & original & 80.3\scriptsize{±0.3} & 80.8\scriptsize{±0.4} & 80.4\scriptsize{±0.4} & 80.4\scriptsize{±0.6} & 74.8\scriptsize{±0.8} & \textbf{82.0\scriptsize{±0.3}} & 81.3\scriptsize{±0.7} & 82.2\scriptsize{±0.3} & \textbf{83.6\scriptsize{±0.4 }} \\
    ~ & w/o retrain(our) & \textbf{81.6\scriptsize{±0.4}} & \textbf{81.2\scriptsize{±0.5}} & \textbf{81.1\scriptsize{±0.5}} & 80.9\scriptsize{±0.8} & \textbf{77.1\scriptsize{±0.8}} &\textbf{ 82.0\scriptsize{±0.2}} & \textbf{81.6\scriptsize{±0.6}} & \textbf{82.6\scriptsize{±0.4}} & \textbf{83.6\scriptsize{±0.1}} \\
    ~ & w retrain(our) & \textbf{81.6\scriptsize{±0.4}} & \textbf{81.2\scriptsize{±0.4}} & 81.0\scriptsize{±0.4} & \textbf{81.5\scriptsize{±0.4}} & 77.0\scriptsize{±0.8} & \textbf{82.0\scriptsize{±0.2}} & \textbf{81.6\scriptsize{±0.6}} &\textbf{ 82.6\scriptsize{±0.4}} & 83.5\scriptsize{±0.1} \\
    \hdashline  
    \multirow{3}*{CITESEER} & original & 68.7\scriptsize{±0.5} & 65.6\scriptsize{±1.0} & 65.1\scriptsize{±0.5} & 68.1\scriptsize{±0.7} & 61.6\scriptsize{±1.4} & 69.8\scriptsize{±0.4} & 69.4\scriptsize{±2.1} & \textbf{71.2\scriptsize{±0.3}} & 72.4\scriptsize{±0.6} \\
    ~ & w/o retrain(our) & \textbf{70.0\scriptsize{±0.7}} & 66.8\scriptsize{±1.0} & 67.3\scriptsize{±0.8} & \textbf{70.5\scriptsize{±1.1}} & \textbf{62.1\scriptsize{±1.4}} & 70.6\scriptsize{±0.3} & \textbf{71.3\scriptsize{±0.6}} & 71.0\scriptsize{±0.7} & \textbf{72.7\scriptsize{±0.6}} \\
    ~ & w retrain(our) & 69.8\scriptsize{±0.7} & \textbf{66.9\scriptsize{±1.2}} & \textbf{67.4\scriptsize{±0.8}} & \textbf{70.5\scriptsize{±1.0}} & \textbf{62.1\scriptsize{±1.4}} & \textbf{70.8\scriptsize{±0.2}} & \textbf{71.3\scriptsize{±0.5}} & 71.0\scriptsize{±0.6} & \textbf{72.7\scriptsize{±0.5}} \\
    \hdashline
    \multirow{3}*{PUBMED} & original & 77.9\scriptsize{±0.3} & 77.6\scriptsize{±0.2} & 77.6\scriptsize{±0.4} & 77.4\scriptsize{±0.3} & 75.5\scriptsize{±0.6} & 79.4\scriptsize{±0.1} & 78.6\scriptsize{±0.5} & 78.9\scriptsize{±0.3} & 79.2\scriptsize{±0.2} \\
    ~ & w/o retrain(our) & 78.7\scriptsize{±0.3} & \textbf{78.5\scriptsize{±0.5}} & \textbf{78.5\scriptsize{±0.4}} & 77.8\scriptsize{±0.2} & 76.0\scriptsize{±0.6} & \textbf{80.0\scriptsize{±0.3}} & \textbf{79.3\scriptsize{±0.5}} & \textbf{79.2\scriptsize{±0.2}} & \textbf{79.4\scriptsize{±0.2}} \\
    ~ & w retrain(our) & \textbf{78.8\scriptsize{±0.2}} & 78.4\scriptsize{±0.5} & \textbf{78.5\scriptsize{±0.3}} & \textbf{77.9\scriptsize{±0.2}} & \textbf{76.1\scriptsize{±0.7}} & 79.7\scriptsize{±0.3} & 79.2\scriptsize{±0.5} & 79.1\scriptsize{±0.5} & 79.3\scriptsize{±0.2} \\
    \bottomrule[1.5pt]
\end{tabular}
}
\bigskip\bigskip

\caption{\small{\textbf{Performance of different DropEdge strategies on different datasets and models.} 
    The estimated TopoInf values are calculated with $\mathcal{V}_t$ in Equation \ref{eq:I_A_node_wise} set as the whole nodes.
    The DropEdge rate and the temperature are hyper-parameters chosen from $\{ 0.3, 0.4, 0.5\}$ and $\{ 0.5, 0.75, 1\}$ respectively and are determined by validation.
    We run experiments 10 times and report testing accuracy with a 95\% confidence interval.
    }}
\label{tab:topoinf_guided_dropedge}
\resizebox{.95\columnwidth}{!}{    
\begin{tabular}{l l c c c c c c c c c}
    \toprule[1.5pt]
    Datasets & DropEdge & GCN & SGC & APPNP & GAT & GIN & TAGCN & GPRGNN & BERNNET & GCNII \\
    \midrule
    \multirow{3}*{CORA} & without & 80.3\scriptsize{±0.4} & 80.7\scriptsize{±0.7} & 80.5\scriptsize{±0.5} & 80.6\scriptsize{±0.6} & 74.7\scriptsize{±0.8} & 82.0\scriptsize{±0.4} & 81.5\scriptsize{±1.7} & 82.2\scriptsize{±0.4} & 83.6\scriptsize{±0.5} \\
    ~ & random & 81.9\scriptsize{±0.7} & 81.4\scriptsize{±0.8} & 80.8\scriptsize{±0.7} & 81.7\scriptsize{±0.6} & 76.1\scriptsize{±0.9} & \textbf{83.0\scriptsize{±0.4}} & 82.0\scriptsize{±0.9} & 82.3\scriptsize{±0.2} & 84.2\scriptsize{±0.4} \\
    ~ & TopoInf(our) & \textbf{82.2\scriptsize{±0.6}} & \textbf{81.5\scriptsize{±0.9}} & \textbf{81.6\scriptsize{±0.5}} & \textbf{82.2\scriptsize{±0.5}} & \textbf{77.2\scriptsize{±0.7}} & \textbf{83.0\scriptsize{±0.3}} & \textbf{82.1\scriptsize{±0.8}} & \textbf{82.8\scriptsize{±0.3}} & \textbf{84.5\scriptsize{±0.4}} \\
    \hdashline
    \multirow{3}*{CITESEER} & without & 68.8\scriptsize{±0.7} & 65.9\scriptsize{±1.3} & 65.1\scriptsize{±0.6} & 69.3\scriptsize{±0.9} & 61.4\scriptsize{±1.4} & 69.9\scriptsize{±0.4} & 70.4\scriptsize{±0.7} & 71.2\scriptsize{±0.5} & 72.5\scriptsize{±0.6} \\
    ~ & random & 68.7\scriptsize{±0.9} & 67.1\scriptsize{±1.0} & 67.8\scriptsize{±0.8} & 69.8\scriptsize{±0.8} & 62.4\scriptsize{±1.0} & 71.2\scriptsize{±0.4} & 70.9\scriptsize{±0.7} & 71.1\scriptsize{±0.6} & 72.5\scriptsize{±0.6} \\
    ~ & TopoInf(our) & \textbf{69.3\scriptsize{±1.1}} & \textbf{67.8\scriptsize{±0.9}} & \textbf{68.3\scriptsize{±0.9}} & \textbf{70.3\scriptsize{±0.5}} & \textbf{63.5\scriptsize{±0.9}} & \textbf{71.7\scriptsize{±0.3}} & \textbf{71.1\scriptsize{±0.6}} & \textbf{71.8\scriptsize{±0.4}} & \textbf{72.9\scriptsize{±0.6}} \\
    \hdashline
    \multirow{3}*{PUBMED} & without & 77.9\scriptsize{±0.5} & 77.6\scriptsize{±0.2} & 77.6\scriptsize{±0.4} & \textbf{77.5\scriptsize{±0.5}} & 75.5\scriptsize{±0.7} & 79.5\scriptsize{±0.3} & 78.7\scriptsize{±0.7} & \textbf{78.9\scriptsize{±0.4}} & \textbf{79.2\scriptsize{±0.2}} \\
    ~ & random & 77.5\scriptsize{±0.4} & 77.9\scriptsize{±0.3} & 77.6\scriptsize{±0.5} & 77.0\scriptsize{±0.4} & 76.9\scriptsize{±0.5} & 79.2\scriptsize{±0.3} & 78.5\scriptsize{±0.7} & 78.7\scriptsize{±0.5} & 78.8\scriptsize{±0.3} \\
    ~ & TopoInf(our) & \textbf{78.0\scriptsize{±0.6}} & \textbf{78.3\scriptsize{±0.3}} & \textbf{78.0\scriptsize{±0.4}} & 77.3\scriptsize{±0.4} & \textbf{77.4\scriptsize{±0.7}} & \textbf{79.7\scriptsize{±0.2}} & \textbf{79.3\scriptsize{±0.4}} & 78.8\scriptsize{±0.5} & 79.1\scriptsize{±0.2} \\
    \bottomrule[1.5pt]
\end{tabular}
}
}
\end{sidewaystable}

In addition, we compare our edge removal strategies with other methods. 
We choose two baseline methods, namely Random and AdaEdge.
For Random, we randomly remove edges with equal probability.
AdaEdge \cite{MADGap} divides edges between nodes with the same labels into the negative set and different ones into the positive set. AdaEdge randomly removes edges in each set, while TopoInf removes edges in each set in descending order, sorted by the absolute values of their TopoInf.

Figure \ref{fig:exp3_compare_different_strategies} shows our results.
When an equal number of positive/negative edges are removed, 
the model performance of Random remains unchanged almost, 
while the model performance of AdaEdge and TopoInf increases/decreases distinctly, indicating a correlation between both the homophily metric and our compatibility metric with the model performance.
Moreover, the model performance of TopoInf increases/decreases more significantly than that of AdaEdge. This shows the effectiveness of TopoInf as it not only identifies the direction of influence 
but also measures the magnitude of influence.

\subsection{Validating Estimated TopoInf}

When ground truth labels are generally only available on training sets, we employ an initial training phase using the GNN model to obtain pseudo labels for each node. 
Thereafter, we estimate TopoInf 
using these pseudo labels for the nodes without true labels and remove edges based on the estimated TopoInf. 

\begin{figure}[htbp]
  \centering
  \subfloat     
  {
      \label{fig:exp3_compare_different_strategies_pos}\includegraphics[width=0.5\textwidth]{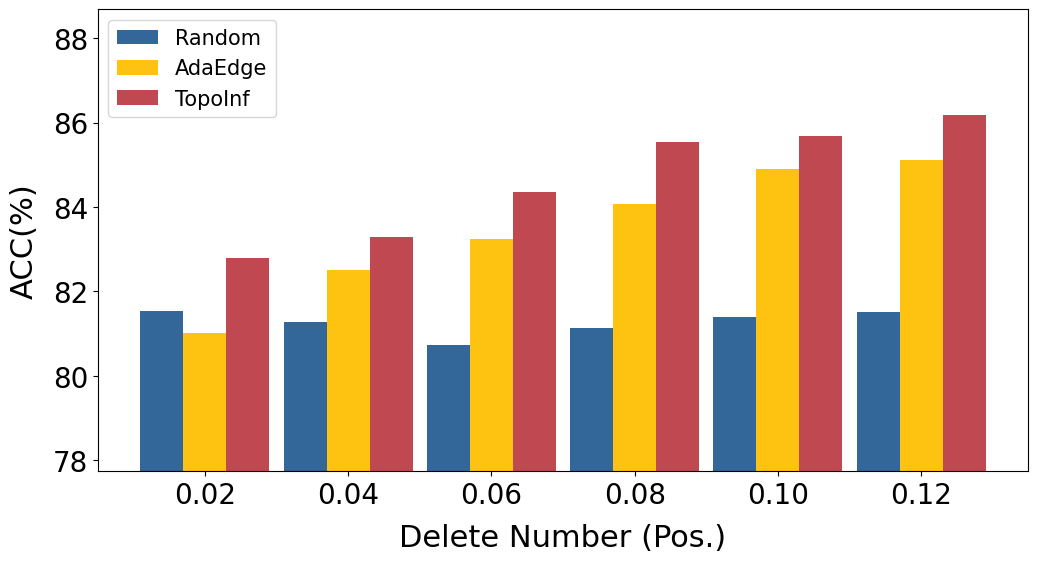}

  }
  \subfloat     
  {
      \label{fig:exp3_compare_different_strategies_neg}\includegraphics[width=0.5\textwidth]{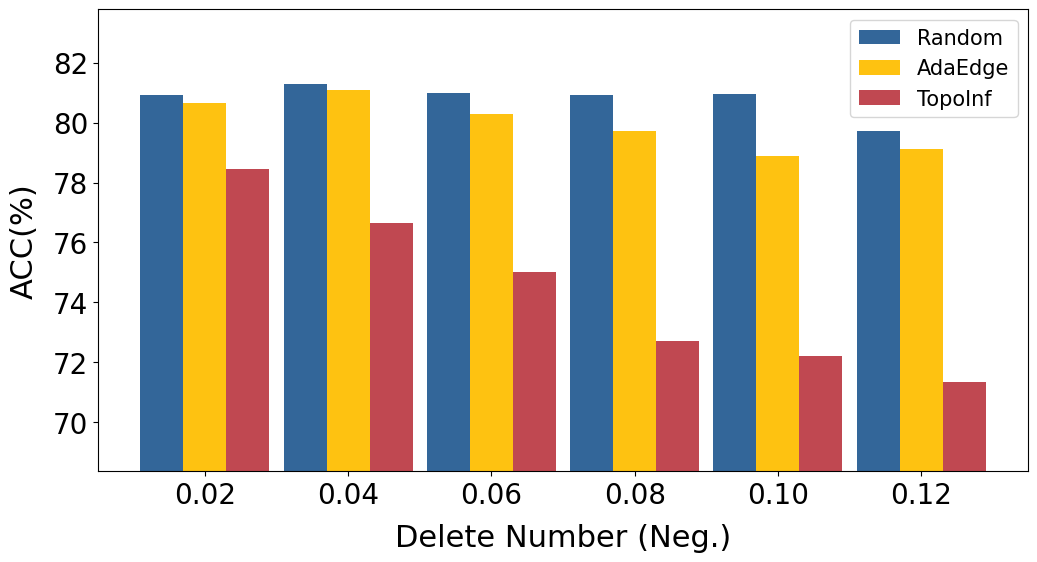}
  }
  \vspace{-3mm}
  \caption{
  \textbf{GCN performance with different edge deletion strategies on Cora.} 
    The left figure shows the results of removing positive edges to increase performance (higher is better), and the right figure shows the results of removing negative edges to decrease performance (lower is better). 
    The horizontal axis denotes the ratio of deleted edges to all edges.
    The vertical axis is the accuracy of the node classification in GCN.
    The TopoInf values are calculated with $\mathcal{V}_t$ in the Equation \ref{eq:I_A_node_wise} set as a test set.
  }
\label{fig:exp3_compare_different_strategies}
\vspace{-6mm}
\end{figure}

\noindent\textbf{How well does estimated TopoInf work to improve model performance?}
We conduct experiments on nine GNN models.
We train the model and obtain pseudo labels.
After that, we can estimate TopoInf based on pseudo labels and refine the original topology based on the estimated TopoInf.
Specifically, we remove a given number of edges, which is a hyper-parameter determined by validation, with the highest estimated TopoInf. 
Since we have already obtained the model parameters through the initial training, we can choose to either continue to apply these parameters on the refined topology (denoted as w retrain) or retrain the model on the refined topology (denoted as w/o retrain).

Table \ref{tab:exp3_compare_model} shows the results.
We can see that in almost all cases, the refined topologies based on estimated TopoInf outperform the original ones, in both retraining and non-retraining settings.
On SOTA models, such as BernNet, TAGCN and GCNII, TopoInf still achieves competitive results or improves performance.
This experiment shows that the estimated TopoInf is still an effective metric for enhancing topology.

\noindent\textbf{Are there any other approaches to utilize estimated TopoInf?}
In this work, we also demonstrate that the estimated TopoInf can be combined with other methods to guide topology modification and improve model performance. 
We take
DropEdge \cite{DropEdge} as an example, where edges are randomly removed from the original graph at each training epoch to help prevent over-fitting and over-smoothing and achieve better performance. 
We replace the random dropping edge scheme with a TopoInf-guided scheme, where edges with higher TopoInf are more likely to be dropped.
Specifically, for TopoInf-guided DropEdge, we define our edge-dropping probability as
$
P_{ij} \propto \textbf{A}_{ij} \cdot \exp \left(  \nabla \mathcal{C}_{\textbf{A}}(e_{ij}) / \tau \right),
$
 where $ \textbf{A}_{ij}$ indicates the existence of $e_{ij}$, $\nabla \mathcal{C}_{\textbf{A}}(e_{ij})$ is the TopoInf of $e_{ij}$, and $\tau$ is the temperature, which is a hyper-parameter and determined by validation.

Table \ref{tab:topoinf_guided_dropedge} shows the results.
In almost all cases, our TopoInf-guided DropEdge method achieves the highest accuracy, and in all cases, our method outperforms DropEdge.
Even in cases where DropEdge degenerates the performance of GNNs, such as BernNet on CiteSeer, and GPRGNN on PubMed, TopoInf-guided DropEdge can still improve the performance.
This experiment shows the potential to combine TopoInf with other methods for topology modification and performance improvement. 

\section{Related Works}
\label{sec:related_works}

\noindent\textbf{Graph filters optimization.}
These studies approach graphs as filters for features and aim to find the optimal filter coefficients for the graph learning task.
ChebNet \cite{ChebNet} approximates the spectral graph convolutions using Chebyshev polynomials. 
GPRGNN \cite{GPRGNN} adaptively learns a polynomial filter and jointly optimizes the neural parameters of the neural network as well as filter coefficients.
ASGC \cite{ASGC} obtains filter coefficients by minimizing the approximation error between the raw feature and the feature filtered by graph topology, to make SGC appropriate in heterophilic settings. 
These works improve graph learning by optimizing coefficients of graph filters, while we focus on optimizing graph topology.

\noindent\textbf{Graph topology optimization.}
Curvature-based methods \cite{OversquashingCurvature} view graphs from a geometric view, measure the effect of edges by the graph curvature, and rewire the graph topology with the help of curvature, to improve the model performance.
Node pair distances \cite{MADGap,measure} and  edge signal-to-noise ratio \cite{ESNR&GPS}
are proposed to measure the topological information and alleviate the over-smoothing to further build deeper models to exploit the multi-hop neighborhood structures. In the era of large models, Sun et al. \cite{llmstructure} utilize LLMs to evaluate the edges between nodes based on node attributes in order to optimize graph topology. 

\section{Conclusion}

In this work, 
we model and analyze the compatibility between the topological information of graph data and downstream task objectives, 
and propose a new metric called TopoInf to characterize the influence of one single edge by measuring the change of the compatibility.
We also discuss the potential applications of leveraging TopoInf to model performance adjustment and empirically demonstrate the effectiveness of TopoInf.
In future, a better estimation of the topological effect of GNNs may further improve our metric.  

\section*{Acknowledgment}
This work was supported by the National Key Research and Development Plan No. 2022YFB3904204,
NSF China under Grant No. 62202299, 62020106005, 61960206002, 
Shanghai Natural Science Foundation No. 22ZR1429100.



\end{document}